\documentclass[a4paper,twoside]{article}

\usepackage{epsfig}
\usepackage{subfigure}
\usepackage{calc}
\usepackage{amssymb}
\usepackage{amstext}
\usepackage{amsmath}
\usepackage{amsthm}
\usepackage{multicol}
\usepackage{pslatex}
\usepackage{SciTePress}
\usepackage[small]{caption}

\usepackage{multirow} % for columns spanning multiple rows
\usepackage{booktabs} % professional-look tables
\usepackage{lineno}

\newtheorem{fact}{Fact}
\newtheorem{definition}{Definition}

\begin{document}

\title{On the Product Rule for Classification Problems}

\author{\authorname{Marcelo Cicconet}
\affiliation{New York University}
\email{cicconet@gmail.com}
}

\keywords{Supervised Learning; Classification; Product Rule.}

\abstract{We discuss theoretical aspects of the product rule for classification problems
in supervised machine learning for the case of combining classifiers.
We show that
(1)~the product  rule arises from the MAP classifier
supposing equivalent priors and conditional independence given a class;
(2)~under some conditions, the product rule is equivalent to minimizing the sum
of the squared distances to the respective centers of the classes related with
different features, such distances being weighted by the spread of the classes;
(3)~observing some hypothesis, the product rule is equivalent to concatenating the vectors of features.}

\onecolumn \maketitle \normalsize \vfill

%----------------------------------------------------------------------------------------------------
\section{Introduction}\label{sec:introduction}
%----------------------------------------------------------------------------------------------------
\vspace{-0.25cm}
With the advance of the Machine Learning field, and the discovery
of many different techniques,
the subject of \emph{combining multiple learners} \cite{Alpaydin2004} eventually
drove attention, in particular the problem of \emph{combining classifiers}.
Many different methods appeared, and soon they
were compared in terms of efficiency in solving problems.

The \emph{product rule} has been present in some of these works
(e.g., \cite{Alexandre2001,Kittler1998,Breukelen1998,Duin2000,Cicconet2010,Cicconet2010B,Li07}),
in contexts ranging from the accuracy of the different combination rules
to some analytical properties of the different methods.

In \cite{Breukelen1998} it was shown that,
in the context of handwritten digit recognition, the product rule performs better
for combining linear classifiers.
In general, however, the product rule does not stand out from competitors \cite{Duin2000}.
For the problem of combining audio and video signals in guitar-chord recognition,
the product rule is better then the sum rule \cite{Cicconet2010},
but on the problem of identity verification using face and voice profiles, the
sum rule wins \cite{Kittler1998}.

On the theoretical realm, \cite{Alexandre2001} shows that
for problems with two classes, the sum and product rules are equivalent
when using two classifiers and the sum of the estimates of the a posteriori probabilities
is equal to one. In \cite{Kittler1998}, the product rule is derived from the hypothesis of
conditional statistical independence between different representations of the data.
There are also some intuitive explanations for the choice of the product rule,
as for instance the fact that the product (``END'' operator) is preferred with respect to the
sum rule (``OR'' operator) because it enforces all qualities defined by the measures at once
\cite{Mertens2007}.

In this text, analytical properties of the product rule are further analyzed,
in the contexts of two or more classifiers.
We show that
(1)~the product  rule arises from the MAP classifier
supposing equivalent priors and conditional independence given a class;
(2)~under some conditions, the product rule is equivalent to minimizing the sum
of the squared distances to the respective centers of the classes related with
different features, such distances being weighted by the spread of the classes;
(3)~observing some hypothesis, the product rule is equivalent to concatenating the vectors of features. 

Our work extends the current theoretical understanding of the product rule provided by Alexandre \emph{et al} \cite{Alexandre2001}
and Kittler \emph{et al} \cite{Kittler1998}, as it was made in the direction of the sum rule by Li and Zong \cite{Li07}.
%Since the product rule occurs frequently in classifier combination contexts, and,
%to the best of our knowledge, the results to be presented are novel and significant,
%we believe this text will be of relevance to the Pattern Recognition community.

%The remaining of the paper is organized as follows:
%in section~\ref{sec:tf} we present theoretical aspects of the product rule;
%experiments related to the time-performance of a non-parallel implementation of the product rule
%are described in section~\ref{sec:experiments};
%conclusions and future-work topics are presented in section~ \ref{sec:cfw}.

%----------------------------------------------------------------------------------------------------
\section{Theoretical Facts}\label{sec:tf}
%----------------------------------------------------------------------------------------------------
\vspace{-0.25cm}
\begin{definition}\label{def:pr}
Let $X,Y$ be (continuous) random variables corresponding to $2$ distinct feature vectors, and
$C$ the (discrete) random variable corresponding to the class, whose output can be $c_1,...,c_K$.
For any $Z \in \{X,Y\}$ and $k \in \{1,\ldots,K\}$, let
$p_{Z,k}$ be a function that outputs the \emph{confidence} that the class
is $c_k$ considering that the features-variable is $Z$.
Supposing that the features are $X = x$ and $Y = y$,
the \emph{product rule} for classification will assign $C = c_{\hat{k}}$ provided
\begin{equation*}
p_{X,\hat{k}}(x) \cdot p_{Y,\hat{k}}(y) = \max_{k = 1,...,K}p_{X,k}(x) \cdot p_{Y,k}(y)\text{ .}
\end{equation*}
\end{definition}

In this definition and in the following results we are using, for simplicity,
only two random variables, named $X$ and $Y$.
We could have used, instead, a set of $N$ random variables, say $X^1,...,X^N$,
but that would unnecessarily overload the notation.

\begin{definition}
Let $(X,Y)$ be the random variable obtained by concatenating
the features $X$ and $Y$, and
$p(\cdot| C = c_k)$
the density function for the variable $(X,Y)$ conditioned to $C = c_k$.
We will denote the value of this function at the point
$(x,y)$ by $p(X = x, Y = y | C = c_k)$.
Let $P(C = c_k)$ be the \emph{prior} probability that
the class is $C = c_k$.

Finally, let us define $p_{(X,Y),k}(x,y)$ as follows:
\begin{equation*}
p_{(X,Y),k}(x,y) = p(X=x,Y=y | C = c_k) \cdot P(C = c_k)\text{ .}
\end{equation*}

Given a sampled value $(X,Y) = (x,y)$,
the \emph{MAP} (Maximum a Posteriori) classifier will assign $C = c_{\hat{k}}$ provided
\begin{equation*}
p_{(X,Y),\hat{k}}(x,y) = \max_{k = 1,...,K}p_{(X,Y),k}(x,y)
\end{equation*}

\end{definition}

\begin{fact}\label{fact:map}
When using the MAP classifier,
the product rule arises under the hypothesis of (1)
conditional independency given the class and (2)
same prior probability for the classes.
\end{fact}
\begin{proof}
The MAP classifier is given by
\begin{equation*}
p(X = x, Y = y | C = c_k) \cdot P(C = c_k)\text{ .}
\end{equation*}
\noindent
Now hypothesis 1 means
\begin{eqnarray*}
&p(X = x, Y = y | C = c_k) =& \\
&= p(X = x | C = c_k)\cdot p(Y = y | C = c_k)\text{ ,}&
\end{eqnarray*}
\noindent
and hypothesis 2 implies that $P(C = c_{\tilde{k}}) = P(C = c_{\hat{k}})$ for all $\tilde{k},\hat{k} = 1,...,K$.
Therefore
\begin{eqnarray*}
&\max_{k = 1,...,K}p_{(X,Y),k}(x,y) =&\\
&= \max_{k = 1,...,K}p(X = x | C = c_k)\cdot p(Y = y | C = c_k)\text{ ,}&
\end{eqnarray*}
which is the product rule (see definition~\ref{def:pr}) for
$p_{X,k}(x) = p(X = x | C = c_k)$ and $p_{Y,k}(y) = p(Y = y | C = c_k)$.

\end{proof}

\begin{fact}
For each $Z \in \{X,Y\}$,
let $d_Z$ be the (finite) dimension of the variable $Z$,
$I_{d_Z}$ the identity matrix of dimensions $d_Z \times d_Z$,
and $\Sigma_{Z,k} = \sigma_{Z,k}^2I_{d_Z}$ (where $\sigma_{Z,k}$ is positive number).
Also, for each $k=1,\ldots,K$, let $\mu_{Z,k}$ be fixed points in $\mathbb{R}^{d_Z}$.

Defining confidence functions (see definition~\ref{def:pr})
\begin{eqnarray}
p_{X,k}(x) = e^{-\frac{1}{2}(x-\mu_{X,k})^\top\Sigma_{X,k}^{-1}(x-\mu_{X,k})}\text{ , and} \\
p_{Y,k}(y) = e^{-\frac{1}{2}(y-\mu_{Y,k})^\top\Sigma_{Y,k}^{-1}(y-\mu_{Y,k})}\text{ ,}
\end{eqnarray}

\noindent
the product rule is equivalent to
\begin{equation*}
\min_{k = 1,...,K} {\frac{1}{\sigma_{X,k}^2}\|x-\mu_{X,k}\|^2
		+\frac{1}{\sigma_{Y,k}^2}\|y-\mu_{Y,k}\|^2}\text{ .}
\label{eq:prod_rule}
\end{equation*}
\noindent
That is, supposing gaussian-like classifiers with covariances parallel to the axis,
the product rule tries to minimize the sum of the squared distances
to the respective ``centers'' of classes for $X$ and $Y$,
such distances being weighted by the inverse of the ``spread'' of the
the classes (an intuitively reasonable strategy, in fact).
\end{fact}
\begin{proof}
Under the mentioned hypothesis, we have
\begin{eqnarray*}
&\max_{k = 1,...,K}p_{X,k}(x) \cdot p_{Y,k}(y) =&\\
&= \max_{k = 1,...,K} e^{-\left(\frac{1}{2\sigma_{X,k}^2}\|x-\mu_{X,k}\|^2
		+\frac{1}{2\sigma_{Y,k}^2}\|y-\mu_{Y,k}\|^2\right)}\text{ .}&
\end{eqnarray*}
Applying $\log$ and multiplying by $2$ the second member of the above equality results in
\begin{eqnarray*}
&\max_{k = 1,...,K}p_{X,k}(x) \cdot p_{Y,k}(y) =&\\
&= \min_{k = 1,...,K} {\frac{1}{\sigma_{X,k}^2}\|x-\mu_{X,k}\|^2
		+\frac{1}{\sigma_{Y,k}^2}\|y-\mu_{Y,k}\|^2}\text{ .}&
\end{eqnarray*}
\end{proof}

\begin{fact}\label{fact:concatenation}
Let us now define confidence functions as follows:
\begin{equation*}
p_{X,k}(x) = \frac{1}{(2\pi)^{d_X}|\Sigma_{X,k}|^{1/2}}e^{-\frac{1}{2}(x-\mu_{X,k})^\top\Sigma_{X,k}^{-1}(x-\mu_{X,k})}\text{ , and} \label{eq:conc:1}
\end{equation*}
\begin{equation*}
p_{Y,k}(y) = \frac{1}{(2\pi)^{d_Y}|\Sigma_{Y,k}|^{1/2}}e^{-\frac{1}{2}(y-\mu_{Y,k})^\top\Sigma_{Y,k}^{-1}(y-\mu_{Y,k})}\text{ ,}\label{eq:conc:2}
\end{equation*}
where, for each $Z\in \{X,Y\}$,
$|\Sigma_{Z,k}|$ is the determinant of $\Sigma_{Z,k}$.
Let us suppose also that, conditioned to the class $c_j$, $X$ and $Y$ are uncorrelated, that is,
being $\Sigma_k$ the covariance of $(X,Y)|C = c_k$, we can write
\begin{equation*}
	\Sigma_k =
	\left[
		\begin{array}{cc}
			\Sigma_{X,k} & 0 \\
			0 & \Sigma_{Y,k}
		\end{array}
		\right] \text{ ,}
\end{equation*}
where, for each $Z \in \{X,Y\}$, $\Sigma_{Z,k}$ is the covariance of $Z | C = c_k$.
Then, putting $\mu_j = (\mu_{X,j}, \mu_{Y,j})$, we have
\begin{eqnarray*}
&p_{X,k}(x)\cdot p_{Y,k}(y) =&\\
&= \frac{1}{(2\pi)^{d_X+d_Y}|\Sigma_k|^{1/2}}
	e^{-\frac{1}{2}((x,y)-\mu_k)^\top\Sigma_j^{-1}((x,y)-\mu_k)} \text{ .}&
\end{eqnarray*}
That is, supposing gaussian classifiers, the product rule is equivalent to learning
using the concatenated vectors of features.
\end{fact}
\begin{proof}
The inverse of $\Sigma_k$ is
\begin{equation*}
	\Sigma_k^{-1} =
	\left[
		\begin{array}{cc}
			\Sigma_{X,k}^{-1} & 0 \\
			0 & \Sigma_{Y,k}^{-1}
		\end{array}
		\right] \text{ .}
\end{equation*}
This way, the expression
\begin{equation*}
	(x-\mu_{X,k})^\top\Sigma_{X,k}^{-1}(x-\mu_{X,k})+(y-\mu_{Y,k})^\top\Sigma_{Y,k}^{-1}(y-\mu_{Y,k})
\end{equation*}
reduces to
\begin{equation*}
	((x,y)-\mu_k)^\top\Sigma_k^{-1}((x,y)-\mu_k) \text{ .}
\end{equation*}
Now
\begin{equation*}
	\frac{1}{(2\pi)^{d_X}|\Sigma_{X,k}|^{1/2}} \cdot \frac{1}{(2\pi)^{d_Y}|\Sigma_{Y,k}|^{1/2}}
	=
	\frac{1}{(2\pi)^{d_X+d_Y}|\Sigma_k|^{1/2}} \text{ .}
\end{equation*}
Therefore
\begin{eqnarray*}
	&p_{X,k}(x) \cdot p_{Y,k}(y) =&\\
	&= \frac{1}{(2\pi)^{d_X+d_Y}|\Sigma_k|^{1/2}}
	e^{-\frac{1}{2}((x,y)-\mu_k)^\top\Sigma_k^{-1}((x,y)-\mu_k)} \text{ .}&
\end{eqnarray*}
\end{proof}

%----------------------------------------------------------------------------------------------------
\section{Discussion}
%----------------------------------------------------------------------------------------------------

According to Fact~\ref{fact:map}, the product  rule arises when maximizing
the posterior under the hypothesis of
equivalent priors and conditional independence given a class.
We have just seen (Fact~\ref{fact:concatenation}) that, supposing only uncorrelation
(which is less then independency), the product rule appears
as well. But in fact we have used gaussian classifiers,
i.e., we supposed the data was normally distributed.
This is in accordance with the fact that
normality and uncorrelation implies independency.

An important consequence of Fact~\ref{fact:concatenation} has to do with the \emph{curse of dimensionality}.
If there is strong evidence that the conditional joint distribution of $(X, Y)$ given any class $C = c_k$
is well approximated by a normal distribution, and that $X|C = c_k$ and $Y |C = c_k$ are uncorrelated, than the product rule is an interesting option, because we do not have to deal with a feature vector with dimension larger the largest of the dimensions of the original descriptors. Besides, the product rule allows parallelization.

%\section*{\uppercase{Appendix}}
%
%\noindent If any, the appendix should appear directly after the
%references without numbering, and not on a new page. To do so please use the following command:
%\textit{$\backslash$section*\{APPENDIX\}}

\vfill
\end{document}